\documentclass[twoside]{article}

\usepackage[preprint]{aistats2026}

\usepackage[round]{natbib}

\usepackage{hyperref}       
\usepackage{refcount}
\usepackage{url}            
\usepackage{booktabs}       
\usepackage{amsfonts}       
\usepackage{nicefrac}       
\usepackage{microtype}      
\usepackage{xcolor}         
\usepackage{lmodern}
\usepackage{mathtools}
\usepackage{wasysym}
\usepackage{tensor}
\usepackage{pgfplots}
\usepackage{enumitem}
\usepackage{makecell}
\usepackage{titlesec}
\usetikzlibrary{calc}
\tikzset{
    myarrow2/.style={-stealth,shorten >=3pt,shorten <=3pt},
  myarrow/.style={stealth-,shorten >=3pt,shorten <=3pt},
  equiv/.style={
    black,
    very thin
  },
}
\pgfplotsset{
  lineplot/.style={
    black,
    dashed,
    thin,
    samples y=0
  },
  coordinate line/.style={
    black,
    samples y=0
  },
  fakeline/.style={
    black,
    opacity=0,
    samples y=0
  },
  point/.style={
    only marks,
    mark=*,
    black,
    mark size=0.5pt
  }
}
\usepgfplotslibrary{colormaps} 
\usetikzlibrary{pgfplots.colormaps}
\usetikzlibrary{positioning}
\usepackage{changepage}
\pgfplotsset{compat=1.18}
\usetikzlibrary{pgfplots.groupplots}
\usepackage[symbol*]{footmisc}

\usepackage{amsthm}
\newtheorem{assumption}{Assumption}
\newtheorem{theorem}{Theorem}
\newtheorem*{remark}{Remark}
\usepackage{amsbsy}
\usepackage{amssymb}
\usepackage{amsmath}
\newcommand{\angles}[1]{\left\langle#1\right\rangle}

\newcommand{\rta}{\rightarrow}
\newcommand{\norm}[1]{\left\lVert#1\right\rVert}
\newtheorem{defn}{Definition}

\usepackage{graphicx}
\usepackage[ruled]{algorithm2e}
\usepackage{bm}
\DeclareMathOperator{\grad}{grad}
\newcommand\scalemath[2]{\scalebox{#1}{\mbox{\ensuremath{\displaystyle #2}}}}
\begin{document}

%

%

\twocolumn[

\aistatstitle{Aggregation on Learnable Manifolds for Asynchronous Federated Optimisation}
\runningauthor{Archie Licudi, Anshul Thakur, Soheila Molaei, Danielle Belgrave \& David Clifton}
\aistatsauthor{ Archie Licudi\textsuperscript{\hspace{1px}\normalfont 2,1} \And Anshul Thakur\textsuperscript{\hspace{1px}\normalfont1} \And Soheila Molaei\textsuperscript{\hspace{1px}\normalfont1} \And Danielle Belgrave\textsuperscript{\hspace{1px}\normalfont1,3} \And David Clifton\textsuperscript{\hspace{1px}\normalfont1} }
\vspace{0.5em}
\aistatsaddress{ \textsuperscript{1}Department of Engineering Science \\
  Oxford University\And\textsuperscript{2}Department of Computing \\
  Imperial College London \And \textsuperscript{3}GlaxoSmithKline }]

\begin{abstract}


    Asynchronous federated learning (FL) with heterogeneous clients faces two key issues: curvature-induced loss barriers encountered by standard linear parameter interpolation techniques (e.g. FedAvg) and interference from stale updates misaligned with the server’s current optimisation state. To alleviate these issues, we introduce a geometric framework that casts aggregation as curve learning in a Riemannian model space and decouples trajectory selection from update conflict resolution. Within this, we propose \textsc{AsyncBezier}, which replaces linear aggregation with low-degree polynomial (Bézier) trajectories to bypass loss barriers, and \textsc{OrthoDC}, which projects delayed updates via inner-product–based orthogonality to reduce interference. We establish framework-level convergence guarantees covering each variant given simple assumptions on their components. On three datasets spanning general-purpose and healthcare domains, including LEAF Shakespeare and FEMNIST, our approach consistently improves accuracy and client fairness over strong asynchronous baselines; finally, we show that these gains are preserved even when other methods are allocated a higher local compute budget.
\end{abstract}

\footnotetext{Preliminary work. Under review by AISTATS 2026.}

\section{INTRODUCTION}

In recent years, Federated Learning (FL) has seen a wave of research interest \citep{FLsurvey1,FLsurvey2} for its ability to keep data in private silos and achieve collaborative model training without the divulgence of centralised data. This has been particularly notable in the healthcare sector \citep{HealthFL,HealthFL1,HealthFL2}, where balancing evolving legislation around the privacy of sensitive data and the performance of models with high-stakes outcomes is a priority. In particular, FL studies optimisation problems of the form:
\begin{equation}
    \min_{\Theta \in \mathcal{M}^\Theta} \mathcal{L}(\Theta) := \frac{1}{M}\sum^M_{i = 1} w_i \mathbb{E}_{(X,y) \sim p_i}[\ell(y;X, \Theta)]
\end{equation}
For some vector of client weights $\textbf{w} \in \mathbb{R}^M$ and some set of client risk functions $\mathcal{L}_i$, corresponding to the expected value of loss $\ell$ over the client data distribution $p_i$. Each client has access only to $\mathcal{L}_i$ and must collaboratively find a minimum $\Theta$, accomplished in the early \textsc{FedAvg} algorithm by a simple arithmetic mean of client models trained by SGD \citep{fedavg}.

Where clients have differing dataset sizes or computational resources, it is often the case that some participants will consistently compute training steps faster than others \citep{compimbalance}, leading to long idle times in the synchronous \textsc{FedAvg} paradigm. This motivates consideration of \textit{asynchronous updates} \citep{fedasync}, where clients are able to submit their results and receive an updated global model to continue training immediately. In this setting, distributional heterogeneity between client datasets poses a more severe challenge as conflicting updates cannot be dealt with synchronously. Despite this, most FL systems in use today rest on the assumption that the linear interpolation of client models produces a strong multi-task model. In the irregular and non-convex loss landscapes of neural networks \citep{landscapeviz}, this assumption can fail as ``barriers'' of higher loss are encountered when averaging along straight lines.

\paragraph{Related Work} There have been many proposals since to mitigate the effects of client heterogeneity and asynchronous update staleness. \cite{fedprox} is a notable example, which adds a \textit{proximal} $L^2$ regularisation term to the client losses; this principle is used in the asynchronous setting by \cite{fedasync}. \cite{fedbuff} takes the simple step of buffering updates to increase training stability, where \cite{asyncfeded} aims to homogenise clients by scaling the number of local epochs each client performs according to the delay with which its updates are received, as well as down-weighting the contribution of updates according to this metric-based ``staleness'' value. Unlike the previous, \cite{dcasgd} directly modifies the update rule, using an approximation to the first-order Taylor expansion of the gradient at the up-to-date point, given the stale gradient. A number of literature proposals are based on adaptive optimisation at the server-side \citep{fadas,fedopt} and seek to delay-correct these momentum terms \citep{ormo, fadas}, but they maintain the same linear connectivity assumption as the aforementioned.

Where methods do make explicit consideration of mode connection geometry, it is usually either indirectly via flatness-aware minimisation \citep{fedsmoo} or whole manifold learning \citep{floco}, neither of which tackle loss barriers explicitly. A final approach which seeks to improve the linear connection quality is \cite{fedma}, performing neuron alignment \citep{mcneuronalignment} before aggregation to factor out permutation equivariance in layers; we find, however, that the number of epochs which each client trains for in the standard federated setting almost never leads to misaligned models, suggesting that this is only appropriate for the direct model fusion problem \citep{li2023deepmodelfusionsurvey}.
\paragraph{Our Contributions} We present a novel family of algorithms in full Riemannian generality \citep{riemanniansgd, Bonnabel_2013, rfedsvrg} that relaxes this linear assumption to the existence of an arbitrary low-loss geodesic; marking a departure from prior art, these ``aggregation manifolds'' are dynamically learned in a modified two-step training process, for which we provide a framework-level convergence result. From these foundations, we propose the \textsc{AsyncBezier} algorithm for asynchronous optimisation as a simple implementation where polynomial mode connections are directly learned as low-loss 1-manifolds and the novel \textsc{OrthoDC} staleness correction rule is deployed to factor out update directions which conflict with the global optimisation trajectory. Finally, we implement a comprehensive empirical testing suite using an asynchronous fork of the Flower FL library \citep{flower}, demonstrating that our proposal is able to consistently outperform existing literature baselines on the canonical benchmark datasets FEMNIST, LEAF Shakespeare, and CXR8.
\raggedbottom
\titlespacing*{\sect}{0pt}{*1}{*1}
\section{BACKGROUND}
\subsection{Mode Connectivity}
Different local minima (\textit{modes}) in parameter space are often connected by simple polynomial curves of low average loss, revealing a large, highly-connected subspace of good solutions \citep{garipovmode,lubana}. These polynomial mode connections often exist between heterogenous multi-task models even where the linear connection fails, and are consistently able to find paths of lower average loss, suggesting natural curvature to this solution subspace \citep{zhouetal}. 

\begin{figure}[t]
    \centering
    \includegraphics[width=0.8\linewidth]{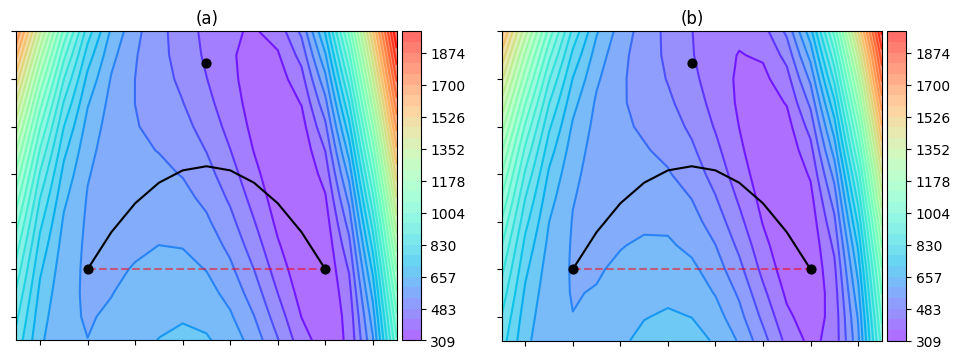}
    \caption{Quadratic Bezier mode connections learned during the federated training of LeNet-5, projected onto a 2-d loss landscape. Plot (a) shows cross-entropy loss w.r.t. a local training set and (b) w.r.t. the global test set.}
    \label{fig:simplexes1}
\end{figure}

Figure \ref{fig:simplexes1} shows the advantage of taking into account curvature and learning quadratic mode connections via a control point orthogonal to the linear connection. In both cases, we see a configuration reminiscent of figures in \cite{garipovmode}, where the longer local training time has allowed the optimisation trajectory to navigate around an ``obstruction'' in parameter space of higher loss that is encountered when moving along the linear connection, but is avoided by the quadratic curve. Work such as \citep{garipovswa, pswa} has examined the positive relationship between choosing models from the midpoint of mode connections and the flatness of minima, conjectured to be correlated with a model's generalisation ability \citep{generalflatness, samimprovinggeneralisation}. 

\subsection{Riemannian Optimisation Preliminaries}

We begin by briefly recalling the key mathematical components of Riemannian Gradient Descent \citep{Bonnabel_2013}:

\begin{defn}[Riemannian Gradient]
    Let $f : \mathcal{M} \rta \mathbb{R}$ be a real-valued $C^\infty$ function w.r.t. a Riemannian manifold $\mathcal{M}$. Then we write $\grad f(x) \in T_x\mathcal{M}$ to denote the unique tangent such that, for all $v \in T_x\mathcal{M}$
    \begin{equation}
        Df_x(v) = \angles{\grad f(x), v}
    \end{equation}
\end{defn}

\begin{defn}[Exponential Map]
    Letting $\gamma_v$ denote the unique geodesic from $x$ with initial tangent vector $v$, we define the Riemannian exponential map:
    \begin{equation}
        \exp_x(v) := \gamma_v(1)
    \end{equation}
\end{defn}
This generalises the idea in Eucldiean space of stepping along a straight line towards a point to R-manifolds. Since geodesics are constant speed, we have the desirable quality that $d(x, \exp_x(v)) \equiv \norm{v}$ where $d$ denotes the induced Riemannian metric on $\mathcal{M}$.

\begin{defn}[Metric-Preserving Transport]
    Letting $x,y \in \mathcal{M}$ we write $P_{x \to y} : T_x\mathcal{M} \to T_y\mathcal{M}$ to denote the \textbf{parallel transport} map with respect to the Levi-Civita connection. This map has the \textbf{(Riemannian) metric-preserving} property:
    \begin{equation}
        \forall v,w \in T_x\mathcal{M}, \;\;\; \angles{P_{x \to y}[v], P_{x \to y}[w]}_x = \angles{v,w}_y
    \end{equation}
\end{defn}
The technical definition of parallel transport in general terms is beyond the scope of this paper, as this property is the only one we actively use (along with the guaranteed existence of such a map for any $x,y \in \mathcal{M}$). It should be noted that $P_{x \to y}$ is not always the only function with this property - it is, however, the only one which also introduces no \textbf{torsion} to the underlying manifold \citep{lee2006riemannian}.

Riemannian GD then proceeds with a simple generalisation of the Euclidean GD update rule:
\begin{equation}
    \theta^{t + 1} \gets \exp_{\theta^t}(\eta \grad(\theta^t)) 
\end{equation}
For some learning rate $\eta \in (0,\infty)$. It is clear how this can be used to generalise Euclidean \textsc{FedAvg} to the Riemannian context, and we can similarly lift the two main paradigms of handling asynchronicity to manifolds. More precisely, the issue of $\grad$ being computed against $\theta^\tau$ for $\tau < t$ can be solved by trusting the learned \textit{position} or \textit{tangent}, exemplified by \textsc{FedAsync} \citep{fedasync} and ASGD \citep{asgd} respectively. We can express these in general Riemannian terms, letting $g^\tau$ denote the learned stochastic pseudo-gradient and $\hat{\theta}^\tau := \exp_{\theta^\tau}(g^\tau)$ the learned model:
\begin{align}
    \theta^{t + 1} &\longleftarrow \exp_{\theta^t}(\eta \exp^{-1}_{\theta^t}(\hat{\theta}^\tau)) \tag{\textsc{AsyncPos}} \\
    \theta^{t + 1} &\longleftarrow \exp_{\theta^t}(\eta P_{\theta^\tau \to \theta^t}[g^\tau]) \tag{\textsc{AsyncTan}}
\end{align}
Other ``delay correcting'' update rules may be lifted to the Riemannian case where there assumptions have non-Euclidean counterparts, such as DC-ASGD:
\begin{align}
    \theta^{t + 1} &\longleftarrow \exp_{\theta^t}\left(\eta P_{\theta^\tau \to \theta^t}[g^\tau + \text{Hess}f(x)[\exp^{-1}_{\theta^\tau}(\theta^t)] \right) \nonumber
\end{align}
The outer product of tangent vectors as an unbiased estimator for the Hessian trick used in the original Euclidean formulation can also be applied to our Riemannian version since the operation occurs in tangent space. In Euclidean space, this ``stepping vector'' can be expressed as a simple linear combination of the ones for \textsc{AsyncPos} and \textsc{AsyncTan}, but this necessitates flatness of the underlying manifold. Due to the variety of update rules proposed in the literature, in the next section we will black-box the function which takes $g^\tau$ as input and outputs a staleness corrected tangent direction for the general framework, before proposing a new geometric rule for \textsc{AsyncBezier}.

\section{THE ASYNCMANIFOLD FRAMEWORK}

We may define the ``aggregation problem'' of AsyncFL as finding the path in parameter space $\gamma : [0,1] \to \mathcal{M}_\Theta$ between the local and global models and the step size $\eta_g \in [0,1]
$ such that $\gamma(\eta_g)$ is in a low point of both the local and global loss landscapes. The most common paradigm for choosing $\gamma$ is the \textit{Linear Mode Connectivity} hypothesis: independent neural network minima are often connected by straight lines of low-loss, so $\gamma$ is simply the straight line $\Theta^\text{local} \leftrightarrow \Theta^\text{global}$. This assumption often fails to hold, however, although minima may still be connected by polynomial curves \citep{lubana}. Some authors consider a stronger hypothesis that extends to entire low-loss submanifolds connecting more than two minima \citep{simplicialcomplexes}, but these approaches based on flat simplicial complexes can encounter the same problem of loss barriers. Instead we make a more immediate generalisation of straight-line connectivity to the Riemannian context that both allows for dynamic adaptation to the solution space geometry and maintains the semantic richness of a manifold learning framework: that there exists a (low-loss) submanifold of $\mathcal{M}_\Theta$ on which the geodesic connection of minima is low-loss. In particular, this subsumes the \textit{Polynomial Mode Connectivity} hypothesis, as we notice that the graph of a Bezier curve is a 1-dimensional submanifold, on which the geodesics trivially follow the polynomial in $\mathbb{R}^\Theta$. An important class of manifolds where the geodesics coincide with a polynomial curve but maintain the dimensionality of $\mathcal{M}_\Theta$ are the \textbf{$\varepsilon$-tunnels} \citep{tunnels}: $\varepsilon$-balls extruded along a B\`ezier curve. This enables a variant of \textbf{Sharpness-Aware Minimisation (SAM)} \citep{samimprovinggeneralisation} for curve learning, which seeks to improve generalisation ability by increasing solution volume.

With the aggregation problem cast as curve learning, we may now present our proposed solution. We specify the \textsc{AsyncManifold} family of algorithms, where the learned aggregation manifold is arbitrary, and provide a particular implementation in \textsc{AsyncBezier}, where we directly learn geodesics as
(quadratic) Bezier curves; finally, we provide a convergence result for the framework, agnostic to the choice of manifold.

\paragraph{\textsc{Training Step (Client)}} Given a particular global model $\Theta^t$, the goal of the client is to learn a submanifold (with boundary) $\mathcal{M}_\phi$ of parameter space $\mathcal{M}_\Theta$ around $\Theta^t$. Our framework is based on the key observation that we can learn a wide class of submanifolds with usual gradient-based methods by choosing a smooth manifold (with boundary) $\mathcal{M}$ and learning a smooth map $\iota_\phi : \mathcal{M} \rightarrow \mathcal {M}_\Theta$, inducing a Riemannian structure on $\mathcal{M}$ by pulling back the metric along the embedding. We call $\mathcal{M}$ equipped with a metric depending smoothly on $\phi$ the Riemannian manifold $\mathcal{M}_\phi := (\mathcal{M},g_\phi)$.

\begin{figure}[h]
    \centering
    \vspace{-1em}
    \includegraphics[width=1.1\linewidth]{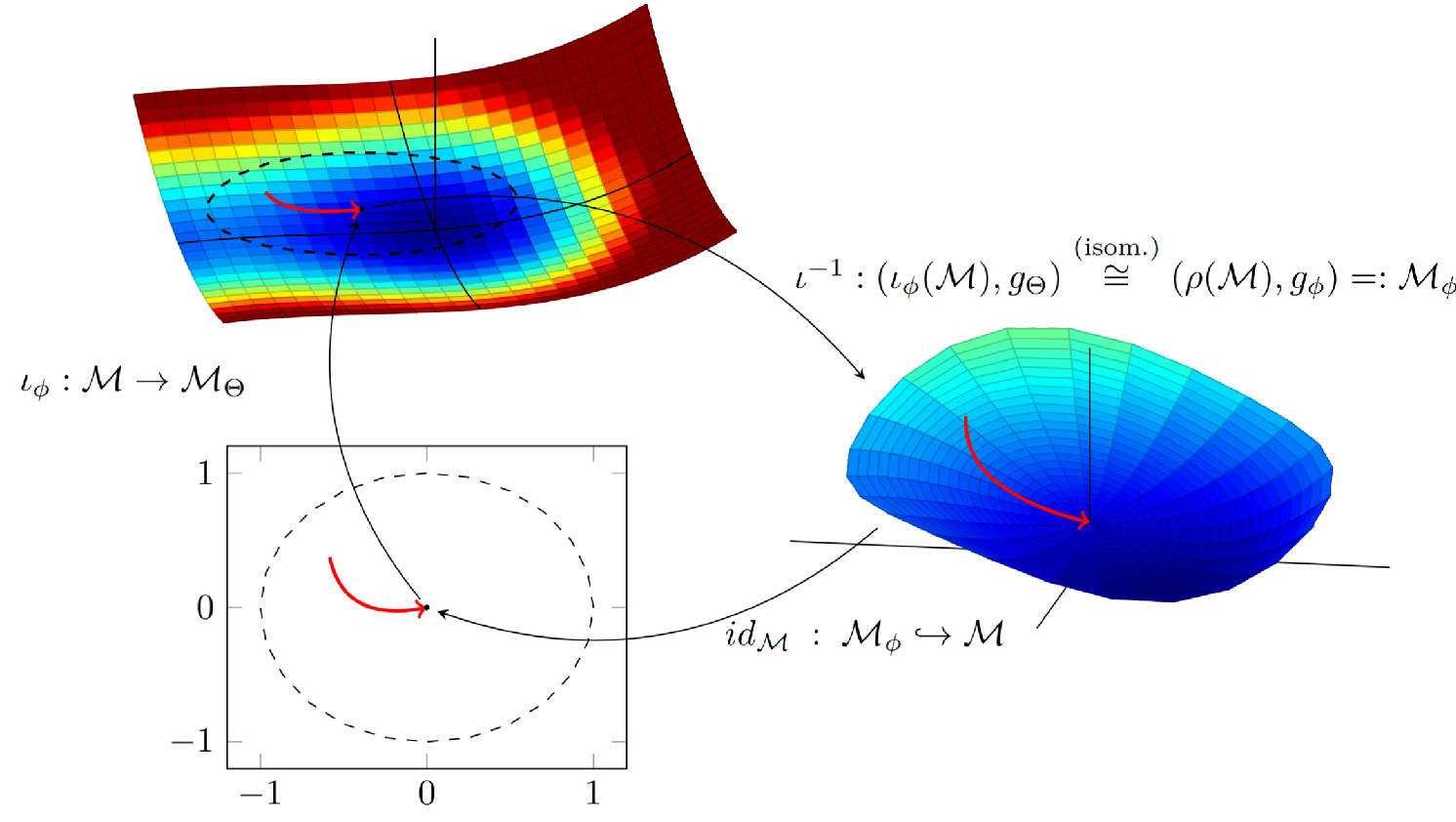}
    \caption[]{Illustration of our approach to manifold learning. $\mathcal{M} = D_1(\mathbb{R}^2)$ maps into parameter space $\mathcal{M}_\Theta = \mathbb{R}^3$ by the learned embedding. $\iota_\phi(\mathcal{M})$ inherits a Riemannian structure from $\mathcal{M}_\Theta$ via the subspace metric, distorted by the to the loss-minimising nature of $\iota_\phi$, which is in turn isometric to a retraction of $\mathcal{M}$ equipped with the pullback metric (in this illustration, the retraction $\rho = id$). The curvature of this $\mathcal{M}_\phi$ space thus induces a lower-loss curved path in $\mathcal{M}$, and hence $\mathcal{M}_\Theta$ under the embedding. \footnotemark}
    \label{fig:enter-label}
\end{figure}

We learn parametrised realisations of $\mathcal{M}$ in $\mathcal{M}_\Theta$ by choosing a smooth map $\iota : \mathcal{M}_\Phi \times \mathcal{M} \rta \mathcal{M}_\Theta$, for some R-manifold $\mathcal{M}_\Phi$. This $\iota$ has two important features: first, for every $\Theta \in \mathcal{M}_\Theta$, there exists a unique $\phi_\Theta \in \mathcal{M}_\Phi$ such that $\iota_{\phi_{\Theta}}(\mathcal{M}) = \{\Theta\}$ - inducing a subspace $\mathcal{M}_\Phi^0$ homoemorphic to $\mathcal{M}_\Theta$. This ``compression'' property is necessitated by the pointwise FL optimisation state being members of $\mathcal{M}_\Theta$ - in order to learn a full a low-loss manifold, we need simply to choose $\mathcal{M}_\Phi$ as the parameter space. Second, $\iota_\phi$ should be an immersion wherever $\phi \not\in \mathcal{M}_\Phi^0$ - this ensures that the pullback metric from $\mathcal{M}_\Theta$ will always induce a Riemannian structure on $\mathcal{M}_\phi$ as soon as the local and global models diverge. Where $\iota_\phi$ is not injective, we will abuse notation and write $\iota_\phi^{-1}(\Theta)$ to mean any member of the $\Theta$ preimage. 

We may now optimise this embedding using standard Riemannian SGD on $\mathcal{M}_\Phi$. For this, we must choose a sampling distribution $\textbf{P}$ over $\mathcal{M}$ which approximates the uniform distribution on the geodesic connecting $\iota_\phi^{-1}(\Theta^t)$ to the distinguished \textit{local model} $\omega \in \mathcal{M}$. Starting from $\Phi_{\Theta^t}$ for the received global model $\Theta^t$, $\phi$ is then trained against the objective:
\begin{align}
    \min_\phi \mathbb{E}&_{S \sim \textbf{P}}\left[F_i(X; \iota_\phi(S), \Theta^t)\right] \;\; := \\ & \;\; \min_\phi \mathbb{E}_{S \sim \textbf{P}}\left[\ell_i(X; \iota_\phi(S)) + \frac{\mu}{2}\norm{\iota_\phi(S) - \Theta^t}^2\right] \nonumber
\end{align}
Optimisation proceeds by general Riemannian gradient descent on $\mathcal{M}_\Phi$, sampling $S_k \sim \textbf{P}_k$ at local batch $k$ - this is possible by the smoothness of the cost function on $\mathcal{M}_\Theta
$ and the smooth immersivity of $\iota$. After $K$ total rounds of optimisation, the reparametrisation vector
\begin{equation}
    v^t_i \in T\mathcal{M}_\Phi := \left(\exp_{\phi_{\Theta^t}}\right)^{-1}(\phi^K)
\end{equation}
is transmitted back to the server.

\begin{remark}
    To perform stochastic analysis we must, separately to any differentiable structure, endow $\mathcal{M}$ and $\mathcal{M}_\Theta$ with probability measures. $\iota_\phi$ must be measurable with respect to them, but the pushforward and latent measures on $\iota_\phi(\mathcal{M})$ need not coincide. In particular, sampling from the uniform distribution on $\iota_\phi(\mathcal{M})$ with respect to the $\mathcal{M}_\Theta$ measure may be possible only by computing a corrected non-uniform distribution on $\mathcal{M}$.
\end{remark} 

\textsc{AsyncBezier} uses the simplest choice of $\iota$ under this framework, learning the aggregation path directly. We choose $\mathcal{M} := [0,1]$ and $\mathcal{M}_\Phi = (\mathbb{R}^\Theta)^{n+1}$ to be the space of control points for degree-$n$ B\`ezier curves in the Euclidean model space $\mathbb{R}^\Theta$. $\iota$ is then defined by de Casteljau's formula, which for the quadratic case is:
\begin{align}
    \iota : (\mathbb{R}^\Theta)^3 \times [0,1] &\longrightarrow \mathbb{R}^\Theta \\
    A,B,C,t &\longmapsto (1-t)^2A + 2t(1-t)B + t^2C \nonumber
\end{align}
Notice that $\iota_\phi$ is thus almost everywhere an embedding. We then fix the parametrisation such that $\iota_\phi(0) = \Theta^t$ and $\omega := 1$. \textbf{P} is set to the Dirac delta at 1 for the first $K_1$ rounds, forcing movement away from the global mode, followed by $\mathcal{U}[0,1]$ for the subsequent $K - K_1$.

\footnotetext{In this figure, we have shown $\mathcal{M}_\Theta$ with Riemannian structure corresponding to the loss landscape for illustration purposes - this will not be the case in general and usually the Riemannian structure of $\mathcal{M}_\Theta$ is defined without $\ell$. Since evaluating the loss function is costly, we induce a new geometry of $\mathcal{M}_\phi$ via distortions in $\iota_\phi$}


\paragraph{Correction Step (Server)} At time step $\tau$, the server receives $v^t_i$ from client $i$. Since $\Theta^\tau$ is out of synchronisation with $\Theta^t$, we need a framework for correcting this staleness. To achieve this,  we fix a function $\pi :  \mathcal{M}_\Theta^2  \times T\mathcal{M}_\Phi \rta T\mathcal{M}_\Phi$, mapping learned gradient and a $(\Theta^t,\Theta^\tau)$ pair to the delay-corrected gradient, ensuring that the $\iota_\phi(\mathcal{M})$ this induces always contains $\Theta^\tau$. We can view this $\pi$ as inducing a weak form of smooth fibre bundle from the total space:
\begin{equation*}
    S_{\phi, \Theta^t} := \left\{ \left(\Theta^\tau, \iota_{(\pi_\phi(\Theta^\tau | \Theta^t)}(x)\right) \mid \Theta^\tau \in \mathcal{M}_\Theta, x \in \mathcal{M}\right\}
\end{equation*}
In particular, for \textsc{AsyncBezier}, $\iota_\phi(\mathcal{M}) \cong \mathcal{M}$ for all $\phi \not\in \mathcal{M}_\Phi^0$, which is true almost everywhere. The ``optimal'' bundle would be one where each $\Theta \in \mathcal{M}_\Theta$ is associated with $\mathcal{M}_\phi$ for the optimal $\phi$, but this would define an intractable $\pi$. Instead, the \textsc{AsyncManifold} framework black-boxes the optimisation (given $\Theta^t$) of the initial client $\phi$ from the perspective of the server and ensure that the transformation to delay-corrected parameters is simple to reason about. This $\pi$ can thus be seen as approximating the true gradient at $\Theta^t$ via the learned geodesic, with convergence guaranteed as long as its error is at most a constant degree worse than the parallel transport of the curve tangent at $\Theta^t$ to $\Theta^s$.

For \textsc{AsyncBezier}, we propose a $\pi$ incorporating a novel delay-correction procedure that directly leverages a general principle of Riemannian geometry: \textit{orthogonality}. One method deployed successfully in multi-task learning is the (sequential) \textbf{Gradient Surgery} approach of \cite{gradientsurgery}. This algorithm considers client update tangents $\Delta_1,\Delta_2$ to ``conflict'' if they have an obtuse angle between them (i.e. $\angles{\Delta_1, \Delta_2}_{\mathcal{M}_\Theta} < 0$). Where updates conflict, $\Delta_1$ will be projected into the orthogonal complement subspace of $\Delta_2$, hence any action of $\Delta_1$ in direct opposition to $\Delta_2$ will be cancelled, whilst preserving orthogonal movement. Inspired by this work, we propose the \textsc{OrthoDC} formula, for tunable hyperparameter $\vartheta \in [-1,1]$ and global drift vector $\Delta^g := \exp^{-1}_{\phi_{\Theta^t}}(\phi_{\Theta^\tau})$:
    \begin{align*}
        \pi(\Theta^t,\Theta^\tau,\Delta) := \begin{cases}\Delta - \text{proj}_{\Delta^g}\left(\Delta\right) & \frac{\angles{\Delta,\Delta^g}}{\norm{\Delta} \cdot \norm{\Delta^g}} \leq \vartheta \\ \Delta & \text{otherwise}\end{cases}
    \end{align*}
Where $\text{proj}_\textbf{b}(\textbf{a}) := \frac{\angles{\textbf{a},\textbf{b}}}{\angles{\textbf{b},\textbf{b}}}\textbf{b}$. Traditional gradient surgery is recovered by setting $\vartheta = 0$, and where $\vartheta = 1$ we only ever consider the orthogonal component of movement. Using $\vartheta = 1$ thus conceptually ``factors out'' the difference between the \textsc{Pos} and \textsc{Tan} approaches on $T_\Theta\mathcal{M}_\Theta$; factoring out the difference in tangent space leads to the approaches coinciding exactly on flat (Euclidean) manifolds, but only up to the first order otherwise. Finally, the server computes
\begin{equation}
    \psi^\tau \longleftarrow \exp_{\phi_{\Theta^\tau}}(\pi(\Theta^t,\Theta^\tau,v^t_i))
\end{equation}

\paragraph{Aggregation Step (Server)} With a final manifold $\mathcal{M}_{\psi^\tau}$ chosen, we find the tangent vector $v^\tau := \exp_{\iota_{\psi^\tau}^{-1}(\Theta^\tau)}(\omega)$ and transition to the next global model by moving part-way along the exponential map. We first define $S^{t,\tau} := 1 + \alpha\left(\norm{\Theta^\tau - \hat{\Theta}^\tau}/\norm{\Theta^t - \Theta^\tau} - 1\right)$ (where $\hat{\Theta}^\tau := \exp_{\Theta^\tau}^{\psi^\tau}\left(v^\tau\right)$) for some decay strength hyperparameter $\alpha \in [0,1]$, and finally define the new global model:
\begin{align}
    \Theta^{\tau + 1} &\longleftarrow \exp_{\Theta^\tau}^{\psi^\tau}\left(S^{t,\tau} \cdot w_{i_\tau}\eta_g^\tau v^\tau\right)
\end{align}
for some global learning rate $\eta^\tau_g \in (0,1]$. This integrates a \textit{staleness penalty}, inspired by \cite{asyncfeded}, to down-weight desynchronised updates. Clients which are perfectly sequential should have an approximately constant $S^{t,\tau}$ (decaying as the gradient magnitude decreases over time), with faster clients being up-weighted and slower ones down-weighted. 

We recall that geodesics are arc-length parametrised and step size in this exponential map is measured according to the $\mathcal{M}_\Theta$ metric pulled back to $\mathcal{M}$. For \textsc{AsyncBezier}, we  achieve this by reparametrisation by simply scaling $\eta^\tau_g$ to ensure that: 
\begin{align}
    \norm{\exp^{-1}_\Theta\left(\iota_\phi\left(\gamma\left(S^{t,\tau}\cdot w_{i_\tau}\tilde{\eta}_g^\tau\right)\right)\right)}_{\mathcal{M}_\Theta} = \\
    \norm{S^{t,\tau}\cdot  w_{i_\tau}\tilde{\eta}_g^\tau\exp_\Theta^{-1}(\iota_\phi(\gamma(1)))}_{\mathcal{M}_\Theta} \nonumber
\end{align}

\paragraph{Meta-Aggregation Step (Server)} Finally, the server may choose to perform \textbf{Stochastic Weight Averaging (SWA)} \citep{garipovswa}, where learning rate schedules are fixed or cyclic and the final returned model is an average of models from throughout the latter stages of the learning process. This is done by Karcher mean on $\mathcal{S}$, the server-side manifold. This can, much like $\mathcal{M}$, be embedded into $\mathcal{M}_\Theta$ \textit{a priori} or by learning a parametric $\iota_{\xi^*}$ such that:
\begin{equation}
    \xi^* = \underset{\xi \in \Xi}{\arg\min} \left[\sum_{t \in A} \min_{x \in \iota_\xi(\mathcal{S})} d_\Theta(\Theta^t, x) \right]
\end{equation}
For some subset of model indices $A \subset [T]$. $d_\Theta$ here denotes any metric on $\mathcal{M}_\Theta$, which may or may not coincide with the induced Riemannian one.

\subsection{Convergence Analysis}
\label{sec:conv}

We may now present our main result on convergence of the framework in general terms; see Appendix \ref{appendix:proof} for precise details of the assumptions made on choice of components.

\begin{theorem}[Convergence of \textsc{AsyncManifold}]
    The \textsc{AsyncManifold} algorithm, with no SWA, assumptions as above, and the local learning rate $\eta_l = \mathcal{O}(1/\max\{2C_1,\sqrt{T}\})$, converges with:
    \begin{align}
        \scalemath{0.93}{\frac{1}{T}\sum^T_{t = 0} \mathbb{E}\norm{\grad \mathcal{L}(\Theta^t)}^2} &\leq \scalemath{0.93}{\mathcal{O}\left(\frac{\lambda_\text{min}}{Q\eta_g\sqrt{T}}\left[\mathcal{L}(\Theta^0) - \mathbb{E}\mathcal{L}(\Theta^T)\right]\right)} \nonumber\\ & \;\;\;\scalemath{0.93}{+ \mathcal{O}\left(\frac{\lambda_\text{min}}{\sqrt{T}}(C_2 + 2C_3)\right)} 
    \end{align}
    Where $C_1,C_2,C_3$ are constants as defined in the proof.
\end{theorem}

\begin{proof}
    See Appendix \ref{appendix:proof} for details.
\end{proof}

\section{EXPERIMENTAL ANALYSIS}

We develop a fork of the Flower FL framework \citep{flower} which handles asynchronous client updates, evaluating \textsc{AsyncBezier} against a number of baseline methods across a variety of datasets.

\subsection{Models and Datasets Used}

We focus on three datasets, each with a different style of task, utilising different model architectures. For full details of each scenario, please see Appendix \ref{sec:expdet}.

\textbf{FEMNIST} \citep{emnist}: The canonical OCR dataset on 62 handwritten characters, using preprocessed versions from the LEAF suite \citep{leaf}. We train a simple 2-conv, 2-dense CNN.

\textbf{Shakespeare} \citep{leaf}: Again from LEAF, performing character-level sequence prediction on a corpus of Shakespeare plays. For this task, we apply a small, 6-head, GPT 2-like \citep{gpt2} transformer.

\textbf{CXR8} \citep{cxr8}: Black-and-white chest X-Ray images, labelled for 8 conditions (including \textit{cardiomegaly} and \textit{pneumothorax}) as a multi-hot vector. We test fine-tuning a ShuffleNet V2 (x1.5) \citep{shufflenet}, using PyTorch's pre-trained ImageNet \citep{imagenet} weights.

The proposed \textsc{AsyncBezier} is then evaluated against 4 representative baselines: \textsc{FedAsync} \citep{fedasync}, DC-ASGD \citep{fedasync}, \textsc{FedBuff} \citep{fedbuff}, and \textsc{AsyncFedED} \citep{asyncfeded}. In addition, to evaluate its influence on our proposal's performance, we implement the standard \textsc{FedAsync} algorithm with the \textsc{OrthoDC} correction rule, terming this \textsc{FedOrtho} where $\vartheta = 1$ and \textsc{FedGS} where $\vartheta = 0$. We differentiate between two versions of our proposed algorithm, with \textsc{AsyncBezierED} using $\alpha = 1$ in the staleness decay parameter and $\alpha = 0$ used otherwise. For the purposes of side-by-side comparison in this paper, we focus only on those methods which are at their core ``SGD-like'' in the update rule, so exclude those proposals which introduce momentum terms and further hyperparameters to tune.

\subsection{Results}

\begin{table*}[t]
  \centerline{
  \begin{minipage}[c]{0.49\textwidth}
    \centering
    \small
    \begin{tabular}{c | c | c}
        \Xhline{2\arrayrulewidth}
        
        \multicolumn{3}{c}{\rule{0pt}{3ex}\textbf{(a) FEMNIST}  \rule[-1.5ex]{0pt}{0pt}} \\
        \Xhline{2\arrayrulewidth}
        \rule{0pt}{3ex}
        \textbf{Method} & \textbf{Test Acc. (\%)} & \textbf{$T_e$}  \rule[-1.5ex]{0pt}{0pt}\\
        \hline
        \rule{0pt}{3ex}
        \textsc{FedAsync} & $85.01 \pm 0.11$ & $137 \pm 6.6$ \\
        \textsc{FedOrtho} & $84.83 \pm 0.08$ & $133 \pm 3.4$ \\
        \textsc{FedGS} & $85.38 \pm 0.14$ & $149 \pm 1.2$ \\
        \textsc{DC-ASGD} & $85.25 \pm 0.17$ & $135 \pm 1.6$ \\
        \textsc{FedBuff} & $84.62 \pm 0.35$ & $174 \pm 2.1$ \\
        \textsc{AsyncFedED} & $85.48 \pm 0.29$ & $114 \pm 5.7$ \\
        \textsc{Proposed} & $\textbf{85.82} \pm \textbf{0.14}$ & $130 \pm 2.6$ \rule[-1.5ex]{0pt}{0pt} \\
        \textsc{ProposedED} & $85.67 \pm 0.14$ & $\textbf{114} \pm \textbf{0.5}$ \rule[-1.5ex]{0pt}{0pt} \\
        \Xhline{2\arrayrulewidth}
    \end{tabular}
  \end{minipage}\begin{minipage}[c]{0.49\textwidth}
    \centering
    \small
    \begin{tabular}{c | c | c}
        \Xhline{2\arrayrulewidth}
         \multicolumn{3}{c}{\rule{0pt}{3ex}\textbf{(b) \textsc{Shakespeare}}  \rule[-1.5ex]{0pt}{0pt}}\\
        \Xhline{2\arrayrulewidth}
        \rule{0pt}{3ex}
        \textbf{Method} & \textbf{Test Acc. (\%)} & \textbf{$T_e$}  \rule[-1.5ex]{0pt}{0pt}\\
        \hline
        \rule{0pt}{3ex}
        \textsc{FedAsync} & $50.60 \pm 0.06$ & $296 \pm 10.0$ \\
        \textsc{FedOrtho} & $52.76 \pm 0.54$ & $202 \pm 14.0$ \\
        \textsc{FedGS} & $52.87 \pm 0.18$ & $209 \pm 11.0$ \\
        \textsc{DC-ASGD} & $52.01 \pm 0.06$ & $230 \pm 8.5$ \\
        \textsc{FedBuff} & $50.84 \pm 0.34$ & $287 \pm 13.0$ \\
        \textsc{AsyncFedED} & $53.03 \pm 0.29$ & $188 \pm 7.0$ \\
        \textsc{Proposed} & $52.07 \pm 0.05$ & $209 \pm 2.0$ \rule[-1.5ex]{0pt}{0pt} \\
        \textsc{ProposedED} & $\textbf{53.13} \pm \textbf{0.13}$ & $\textbf{164} \pm \textbf{2.5}$ \rule[-1.5ex]{0pt}{0pt} \\
        \Xhline{2\arrayrulewidth}
    \end{tabular}
  \end{minipage}
  }
  \vspace{1em}
  \caption{Percentage test set accuracy across methods for the FEMNIST and Shakespeare datasets.}
  \label{fig:shaketable}
\end{table*}

\begin{table*}[t]
  \centering
  \small
  \begin{minipage}[b]{\textwidth}
    \centering
    \begin{tabular}{c | c | c | c}
        \Xhline{2\arrayrulewidth}
         \multicolumn{4}{c}{\rule{0pt}{3ex}\textbf{CXR8 Macros}  \rule[-1.5ex]{0pt}{0pt}}\\
        \Xhline{2\arrayrulewidth}
        \rule{0pt}{3ex}
        \textbf{Method} & \textbf{Test Macro AUROC} & \textbf{Test Macro AUPRC}  & \textbf{$T_e$}  \rule[-1.5ex]{0pt}{0pt}\\
        \hline
        \rule{0pt}{3ex}
        \textsc{FedAsync} & $77.93 \pm 0.01$ & $25.72 \pm 0.21$ & $140 \pm 6.0$ \\
        \textsc{FedOrtho} & $77.91 \pm 0.13$ & $25.90 \pm 0.29$  & $134 \pm 2.5$ \\
        \textsc{FedGS} & $77.85 \pm 0.10$ & $\textbf{26.31} \pm \textbf{0.18}$ & $141 \pm 6.0$ \\
        \textsc{DC-ASGD} & $78.32 \pm 0.39$ & $26.06 \pm 0.34$ & $146 \pm 1.5$ \\
        \textsc{FedBuff} & $77.82 \pm 0.02$ & $25.89 \pm 0.14$ & $172 \pm 2.0$ \\
        \textsc{AsyncFedED} & $77.45 \pm 0.08$ & $25.37 \pm 0.13$ & $144 \pm 2.5$ \\
        \textsc{Proposed} & $\textbf{78.44} \pm \textbf{0.04}$ & $26.12 \pm 0.11$ & $132 \pm 6.8$ \rule[-1.5ex]{0pt}{0pt} \\
        \textsc{ProposedED} & $77.89 \pm 0.12$ & $26.11 \pm 0.12$ & $\textbf{116} \pm \textbf{9.0}$ \rule[-1.5ex]{0pt}{0pt} \\
        \Xhline{2\arrayrulewidth}
    \end{tabular}
  \end{minipage}

  \vspace{1em}
  \caption{Macro AUROC and AUPRC scores for each method across the 8 conditions in the CXR8 dataset.}
  \label{fig:cxrtab}
\end{table*}
Table \ref{fig:shaketable} shows the test set accuracy results for both our proposal and the baseline methods over the Shakespeare and FEMNIST datasets, with Table \ref{fig:cxrtab} showing the macro AUROC and AUPRC results for CXR8. To give an accurate impression the balance between accuracy at convergence and speed to reach a target error level, we choose an error (defined as 1 - AUROC for CXR8) threshold $e$ close to the converged value and report $T_e$, the number of communication rounds at which this threshold is reached.

Each model was trained for 360 communication rounds (720 total epochs, avg 24/client), with $e = 0.20, 0.50, 0.25$ for the FEMNIST, Shakespeare, and CXR8 datasets respectively. Each scenario was repeated with three different random seeds, with the means and standard deviations across runs being reported in the table.

We can make the following observations: \textbf{(1)} The optimal choice of delay-corrected update rule is sensitive to dataset. In particular, we see that different values of $\vartheta$ are optimal for \textsc{AsyncBezier} on different problems, illustrating the ways in which the geometric relationships between clients are task-dependent. \textbf{(2)} \textsc{AsyncBezier} (with optimal choice of $\alpha$) always outperforms \textsc{FedAsync}, with an average +1.05\% performance and -54 epochs to target error. \textbf{(3)} Indeed, our proposal outperforms every other baseline on every metric (by an average +.17\% performance advantage vs. the runner-up with -9 epochs) other than CXR8 AUPRC, where it ranks 3rd behind \textsc{AsyncFedED} and \textsc{FedGS}. The disparity between AUROC and  AUPRC results may be attributed to the difficulty of this task, especially for the lightweight ShuffleNet model, reflected in the poor overall performance of AUPRC scores, with high class imbalance and some conditions significantly harder to detect than others. This still provides a useful benchmark against less well-studied real-world datasets, although future work would evaluate the \textsc{AsyncManifold} method specialised to complex tasks with larger models that can achieve a higher baseline AUPRC score, since solution space geometry may exhibit more stable and transferable characteristics in this case. \textbf{(4)} The proposals based on \textsc{OrthoDC} usually outperform \textsc{FedAsync}, but the gains of \textsc{AsyncBezier} cannot solely be attributed to this since they still consistently have an advantage of an average +.41\% performance and -25 epochs vs. \textsc{FedGS}/\textsc{FedOrtho}. \textbf{(5)} Indeed, our proposal is the only, other than DC-ASGD, which outperforms naive \textsc{FedAsync} on every dataset. Our proposal also outperforms DC-ASGD on every dataset, by an average of .31\% accuracy/AUROC and 13 communication rounds. In general, we can attribute the superior performance to the greater fitness of our \textit{quadratic mode connection} hypothesis to dataset geometries than that of \textit{linear mode connection}. 

\subsubsection{Client Fairness}
\label{sec:fair}

When dealing with both statistical and size heterogeneity in client distributions, it is important to consider the equitable treatment of model performance on each dataset, even where they might be under-represented in the global loss function. We term this desirable property \textit{client fairness} \citep{mohri2019agnosticfederatedlearning}, and it is particularly relevant in the healthcare setting, where clients will often correspond to hospitals with different patient demographics \citep{HealthFL}.

\begin{figure*}[t]
    \begin{center}
    \makebox[\textwidth]{%
    \begin{tikzpicture}
        \begin{groupplot}[
            height=4cm,
            group style={
                group size=2 by 1,
                horizontal sep=1cm,
            },
            ymin=0,
            axis x line*=bottom,
            axis y line*=left,
            legend style={
                at={(0,-0.25)},
                anchor=north west,
                legend columns=4
            },
            legend to name=newlegend
        ]
        
        \nextgroupplot[
            ylabel={Gini Coefficient},
            ybar=0pt,
            bar width=6pt,
            symbolic x coords={FEMNIST, Shakespeare, CXR8},
            xtick=data,
            enlarge x limits=0.25,
            x=2.1cm
        ]
        \addplot+[draw=black!90, fill=black!30, error bars/.cd, y dir=both, y explicit]coordinates {(FEMNIST, 0.02289956) +- (0,0.00157218) (Shakespeare, 0.008691787979070434) +- (0, 0.00001652510228162228) (CXR8, 0.016995022842803322) +- (0, 0.00015780121169222186)};
        \addplot+[draw=brown!90, fill=brown!30, error bars/.cd, y dir=both, y explicit]  coordinates {(FEMNIST, 0.02526614) +- (0,0.00242492) (Shakespeare, 0.009972973522535743) +- (0, 0.00020162100965631126) (CXR8, 0.017032143581216574) +- (0, 0.0021495100356583212)};
        \addplot+[draw=blue!90, fill=blue!30, error bars/.cd, y dir=both, y explicit] coordinates {(FEMNIST, 0.02852245) +- (0,0.00391203) (Shakespeare, 0.008917136419143715) +- (0, 0.000799361541666049) (CXR8, 0.017937807123133365) +- (0, 0.0006355722344925895)};
        \addplot+[draw=red!90, fill=red!30, error bars/.cd, y dir=both, y explicit] coordinates {(FEMNIST, 0.02433821) +- (0,0.00103) (Shakespeare, 0.009450869982946895) +- (0, 0.00028982368674618144) (CXR8, 0.018545900309521093) +- (0, 0.00035447676063269344)};
        \addplot+[draw=pink!90, fill=pink!30, error bars/.cd, y dir=both, y explicit] coordinates {(FEMNIST, 0.02413592) +- (0,0.00127) (Shakespeare, 0.008849960398909479) +- (0, 0.0004695822268311967) (CXR8, 0.017979366648524727) +- (0, 0.0009181168672071395)};
        \addplot+[draw=green!90, fill=green!30, error bars/.cd, y dir=both, y explicit]   coordinates {(FEMNIST, 0.02488981) +- (0,0.00154393) (Shakespeare, 0.00896857142172628) +- (0, 0.00016922723974362262) (CXR8, 0.01807529447391472) +- (0, 0.0005138300188794395)};
        \addplot+[draw=orange!90, fill=orange!30, error bars/.cd, y dir=both, y explicit]  coordinates {(FEMNIST, 0.024212) +- (0, 0.000718) (Shakespeare, 0.008776672185610093) +- (0, 0.000042486470082643384) (CXR8, 0.016995927781812227) +- (0, 0.0016374658055391348)};
        \addplot+[draw=purple!90, fill=purple!30, error bars/.cd, y dir=both, y explicit]  coordinates {(FEMNIST, 0.02516016) +- (0,0.00158741) (Shakespeare, 0.00967133067487672) +- (0, 0.00022800730692326774) (CXR8, 0.017722335685902937) +- (0, 0.0012535731413390316)};

        \nextgroupplot[
            ylabel={Theil Index},
            ybar=0pt,
            bar width=6pt,
            symbolic x coords={FEMNIST, Shakespeare, CXR8},
            xtick=data,
            enlarge x limits=0.25,
            x=2.1cm
        ]
        \addplot+[draw=black!90, fill=black!30, error bars/.cd, y dir=both, y explicit]  coordinates {(FEMNIST, 0.00083045) +- (0,0.00010799) (Shakespeare, 0.00011802526429843854) +- (0, 0.0000012723701492453176) (CXR8, 0.0004457162702341135) +- (0, 0.000008717774647910235)}; 
        \addplot+[draw=brown!90, fill=brown!30, error bars/.cd, y dir=both, y explicit]  coordinates {(FEMNIST, 0.00102926) +- (0,0.00020227) (Shakespeare, 0.00015183396786905387) +- (0, 0.000006962490026986374) (CXR8, 0.0004736665191598803) +- (0, 0.00011816675516907905)};
        \addplot+[draw=blue!90, fill=blue!30, error bars/.cd, y dir=both, y explicit]    coordinates {(FEMNIST, 0.00136546) +- (0,0.00037477) (Shakespeare, 0.00013010717617673682) +- (0, 0.000019248552731944615) (CXR8, 0.0005038596623018956) +- (0, 0.00002868721326809045)};
        \addplot+[draw=red!90, fill=red!30, error bars/.cd, y dir=both, y explicit]     coordinates {(FEMNIST, 0.00094224) +- (0,0.0000723) (Shakespeare, 0.00014111415834195236) +- (0, 0.000008175799613544305) (CXR8, 0.0005357433465720251) +- (0, 0.000025743186828300823)};
        \addplot+[draw=pink!90, fill=pink!30, error bars/.cd, y dir=both, y explicit]     coordinates {(FEMNIST, 0.00093182)+- (0,0.0000935)  (Shakespeare, 0.0001247802937946454) +- (0, 0.0000143328507657223) (CXR8, 0.0005048489556049418) +- (0, 0.00004545543563247114)};
        \addplot+[draw=green!90, fill=green!30, error bars/.cd, y dir=both, y explicit]   coordinates {(FEMNIST, 0.00100771) +- (0,0.00013821) (Shakespeare, 0.0001268099330517064) +- (0, 0.000004600841528695486) (CXR8, 0.000502879457602122) +- (0, 0.000022625417713220886)};
        \addplot+[draw=orange!90, fill=orange!30, error bars/.cd, y dir=both, y explicit]  coordinates {(FEMNIST, 0.00094124) +- (0, 0.0000572) (Shakespeare, 0.00012324139117396714) +- (0, 0.0000016878395292316888) (CXR8, 0.0004603695961885022) +- (0, 0.00008766612048135097)};
        \addplot+[draw=purple!90, fill=purple!30, error bars/.cd, y dir=both, y explicit]  coordinates {(FEMNIST, 0.00103877) +- (0, 0.00014611) (Shakespeare, 0.0001483205747134278) +- (0, 0.0000065044381169181555) (CXR8, 0.0004920343734448567) +- (0, 0.00005594334801342532)};
        
        \legend{\textsc{Proposed}, \textsc{ProposedED}, \textsc{FedAsync}, \textsc{FedOrtho}, \textsc{FedGS}, \textsc{DC-ASGD}, \textsc{FedBuff}, \textsc{AsyncFedED}}
        
        \end{groupplot}
        \end{tikzpicture}
        }
    \end{center}
    \begin{center}
    \vspace{-0.5em}
    \pgfplotslegendfromname{newlegend}
    \end{center}
    \caption{Bar plots of (unweighted) Gini Coefficient and Theil Index computed for each method over the model performance on each client's validation set.}
    \label{fig:fairnessplot}
\end{figure*}

Following \cite{fedmode}, we borrow two classical econometric formulae for calculating the ``inequality'' of a sampled distribution that goes beyond simple variance analysis: the \textit{Gini Coefficient} and \textit{Theil Index} (see Appendix \ref{sec:faircalc}). Figure \ref{fig:fairnessplot} shows these values computed according to the Accuracy/Macro AUROC value distribution for the best performing global model across the decentralised client validation sets; we note that the two metrics broadly agree on the ordering of methods, with the Theil index showing slightly more sensitivity.

There is comparatively little consistent variation amongst the methods, with \textsc{FedOrtho}, \textsc{FedGS}, and \textsc{DC-ASGD} in particular all close together. The -ED variants both show a consistent poorer performance and higher variance than their respective non-scaling counterparts (most noticeable in \textsc{AsyncBezierED}), which is expected due to their intentional down-weighting (to varyig degrees) of certain straggling clients.

Our proposal (with $\alpha = 0$) consistently shows a slight improvement over all other baselines, with an average of $4.7 \times 10^{-4}$ Gini coefficient and $4.0 \times 10^{-5}$ Theil index. We conjecture this may be attributable to the generalisable minima-seeking behaviour of the curve learning process. This shows the clear promise of our framework for applications in the aforementioned medical contexts, along with its strong performance on the CXR8 dataset.

\subsubsection{Effect of Local Epoch Counts}
\label{sec:localepochs}

An important consideration when weighing the use of \textsc{AsyncBezier} is whether the computational overhead from curve-fitting epochs is worth it for the increased communication efficiency, when these epochs could instead be allocated to standard pointwise SGD. To investigate the effect of increased local SGD epochs on final method performance, we re-run FEMNIST training on each of the methods (excluding \textsc{AsyncFedED}) for $T = 360$ communication rounds with each of four different epoch counts. For \textsc{AsyncBezier} we use $\min(K,2)$ curve-fitting epochs when running with $K$ SGD epochs.

\begin{figure}[ht]
    \centering
    
    \begin{tikzpicture}
        \begin{axis}[
            xlabel={Local Epoch Count $K$},
            ylabel={FEMNIST Test Accuracy (\%)},
            xmin=1, xmax=10,
            ymin=82, ymax=87,
            xtick={1,2,5,10},
            ytick={83, 84, 85, 86, 87},
            legend style={
                anchor=south east,
                at={(1, 0)},
                legend columns=1,
                nodes={scale=0.9, transform shape}
            },
            ymajorgrids=true,
            grid style=dashed,
        ]
        \pgfplotsset{
        every axis plot post/.append style={line width=0.75pt}
             }
        \addplot[color=black,  mark=square*]
        coordinates { (1,84.883)(2,85.82)(5,86.22)(10,86.4353) };
        \addplot[color=brown,  mark=square*]
        coordinates { (1,84.9436)(2,85.67)(5,85.8098)(10,85.9924) };
        \addplot[color=blue,  mark=square*]
        coordinates { (1,83.4429)(2,85.01)(5,85.5167)(10,85.5994) };
        \addplot[color=red,  mark=square*]
        coordinates { (1,82.4154)(2,84.83)(5,85.4128)(10,85.5429) };
        \addplot[color=pink,  mark=square*]
        coordinates { (1,84.6382)(2,85.38)(5,85.4504)(10,85.5438) };
        \addplot[color=green,  mark=square*]
        coordinates { (1,84.549)(2,85.25)(5,85.8352)(10,85.6125) };
        \addplot[color=orange,  mark=square*]
        coordinates { (1,82.4129)(2,84.62)(5,85.5004)(10,85.67) };
        
        \legend{\textsc{Proposed}, \textsc{ProposedED}, \textsc{FedAsync}, \textsc{FedOrtho}, \textsc{FedGS}, \textsc{DC-ASGD}, \textsc{FedBuff}}   
        \end{axis}
    \end{tikzpicture}
    \caption{Accuracy of each method on FEMNIST after 360 communication rounds by local epoch count}
    \label{fig:epochs}
\end{figure}

Figure \ref{fig:epochs} shows the results of this investigation. As expected, every method sees mean gains of $1.33\%$ when moving from 1 to 2 epochs and $.45\%$ from 2 to 5, attributable in both to larger step sizes allowing greater progress towards convergence in the fixed $T$. When moving from 5 to 10 epochs, however, the gains for most methods are minimal ($\mu = .09\%$), with DC-ASGD even seeing a decline in accuracy of $.22\%$, attributable to client heterogeneity leading to divergences in the local gradients becoming compounded with the increased time between synchronisation steps. Crucially for this evaluation, our method outperforms every baseline at every $K$ value, with the $K = 2$ version of our proposal outperforming every other method regardless of local epoch count. In particular, it is more efficient to spend 2 epochs in pointwise SGD and 2 epochs in our curve learning procedure (as in the main results of this section) than it is to spend 5 total epochs in pointwise SGD and proceed by any other proposal. Furthermore, our method shows the greatest ability to take advantage of more local epochs, being the only one to reach over 86\% accuracy at higher counts. This suggests an improved capacity to handle divergent local gradients due to our consideration of local solution space geometry.

\section{CONCLUSION AND FUTURE WORK}

In this paper, we have developed \textsc{AsyncBezier}, a new AsyncFL algorithm augmenting SGD-based methods with greater knowledge of client loss landscape geometry, and proven its convergence by situating it within our \textsc{AsyncManifold} Riemannian framework. Our proposal is supported by a novel staleness correction method derived from orthogonal complement projection to minimise conflicting updates from heterogenous clients. In evaluations of both CNN and Transformer architectures on general-purpose and healthcare datasets, our proposal is shown to be empirically superior to strong baselines in terms of both accuracy, AUROC, and fairness. Whilst our method does introduce computational overhead compared to \textsc{FedAsync}, we have shown in Section \ref{sec:localepochs} that our curve learning procedure makes better use of computation budget for higher epoch counts than pure pointwise SGD.

Future work would include deeper analyses of more complex implementations of the \textsc{AsyncManifold} framework, especially on non-Euclidean underlying manifolds, including providing stronger convergence bounds with more specific method-wise assumptions. Applications of \textsc{AsyncBezier} to real-world healthcare contexts would in turn be an important next step from the promising evaluation on the CXR8 dataset, especially where it is deployed on larger clusters with very high numbers of resource-constrained clients or in conjunction with mechanisms for ensuring differential privacy.


\bibliographystyle{unsrtnat}
\bibliography{refs}

\clearpage
\appendix
\thispagestyle{empty}

\onecolumn
\aistatstitle{Aggregation on Learnable Manifolds for Asynchronous Federated Optimisation:
Supplementary Materials}

\section{PROOF OF CONVERGENCE}
\label{appendix:proof}

We begin with the standard assumptions of non-convex optimisation, lifted to the Riemannian context with appropriately adjusted definitions:

\begin{assumption}[L-Smooth Loss]
    \label{ass:lsl}
    There exists a constant $L_\Theta$ such that:
    \begin{equation*}
        \norm{\grad \mathcal{L}(X) - P_{X \to Y}[\grad \mathcal{L}(Y)]} \leq L_\Theta\norm{X - Y}
    \end{equation*}
    For all $X,Y \in \mathcal{M}_\Theta$
\end{assumption}

\begin{assumption}[Bounded Loss Gradient]
    There exists some constant $G$ such that $\norm{\grad\mathcal{L}(\Theta)} \in [0, G]$ for all $\Theta \in \mathcal{M}_\Theta$. The unbiased gradient estimates used for stochastic local steps should also have norm upper bounded by $G$.
\end{assumption}

For simplicity in this paper, we will adopt the following ``weakly homogenous'' setting, which assumes that stochastic gradients w.r.t. $\mathcal{L}_i$ are an unbiased estimator for $\grad \mathcal{L}$.
\begin{assumption}[Unbiased Client Heterogeneity]
\label{ass:boundheterogeneity}
    We have that the local stochastic gradients of the cost function, taken across both the choice of client index and the local entropy during the training step, are unbiased estimators for the global cost. In particular, the expectation of the local stochastic gradient equals the true global gradient.
\end{assumption}
Formally, the cost function in question in the previous assumption is the once whose variance is bounded in:
\begin{assumption}[Bounded Stochastic Divergence from Geodesic]
\label{ass:variance}
    Suppose that local steps at time step $t$ are taken against the cost function:
    \begin{equation}
        \tilde{G}_t(\phi) := \int_{\tilde{\mathcal{M}}_t \subseteq \mathcal{M}} \mathcal{L}(\iota_\phi(x)) \; dp_t(x)
    \end{equation}
    For some probability distribution $p_t$ on $\tilde{\mathcal{M}}_t$, chosen as some subset of $\mathcal{M}$. Then there exists some constants $\sigma_1,\sigma_2$ such that:
    \begin{equation*}
        \mathbb{E}\norm{\grad \tilde{G}_t(\phi) - \grad G(\phi)}^2 \leq \sigma_1^2 + \sigma_2^2\norm{\grad G(\phi)}^2
    \end{equation*}
    Where:
    \begin{equation}
        G(\phi) := \int^0_1 \mathcal{L}(\iota_\phi(\gamma_t(\lambda))) \;  d\lambda
    \end{equation}
    For $\gamma_t$ the geodesic connecting $\iota_\phi^{-1}(\Theta^t) \to \omega$.
\end{assumption}

This modification to the standard bounded stochastic variance assumption seems quite strong on ($n > 1$)-dimensional manifolds, but can be achieved in a number of ways leveraging smoothness and shrinking off-geodesic volume. This is a product of the ``ephemerality'' of the learned manifolds being used to compute steps rather than as part of an effort to learn a low-loss manifold in itself.

Next, we need to bound the reasonableness of functions chosen in the \textsc{AsyncManifold} instantiation:
\begin{assumption}[Lipschitz and Bounded Curvature Embedding]
    \label{ass:lipschitz}
    There exists a constant $M_\Phi$ such that, for all $x,y \in \mathcal{M}$ and $\phi,\psi \in \mathcal{M}_\Phi$:
    \begin{equation}
        \norm{\iota(x,\phi) - \iota(y,\psi)} \leq M_\Phi\norm{(x,\phi) - (y,\psi)}
    \end{equation}
    $\iota$ should also be L-smooth, and from this we have $L$-smoothness of the lifted loss:
    \begin{equation*}
        \hspace{-1em} \norm{\grad (\mathcal{L}\iota)(\phi,x) - P^{(\phi,x)}_{(\psi,y)}[\grad (\mathcal{L}\iota)(\psi,y)]} \leq L_\Phi\norm{X - Y}
    \end{equation*}
    Finally, the operator norm of the second fundamental form (geodesic curvature) of $\iota_\phi$ should be uniformly bounded for any $\phi \in \mathcal{M}_\Phi$ and any $x \in \mathcal{M}$:
    \begin{equation}
        \norm{\mathrm{I\!I}_{\iota_\phi}(x)}_\text{op} \leq C
    \end{equation}
\end{assumption}
\begin{assumption}[Embedding Immersivity]
    \label{ass:immersion}
    $\iota_\omega : \mathcal{M}_\Phi \rta \mathcal{M}_\Theta$ should be an immersion for any $\omega \in \mathcal{M}$. This ensures local injectivity of the differential map, and we furthermore enforce that the smallest eigenvalue of its adjoint is bounded everywhere uniformly above zero by $\sqrt{|\lambda_\text{min}|}$.
\end{assumption}

The following assumption quantifies the ``well-behavedness'' of our delay correction procedure: we should finish with a stepping tangent which is at most a constant times worse as an approximation to $\grad \mathcal{L}(Y)$ than the parallel transport:
\begin{assumption}[Delay Correction Quality]
    \label{ass:quality}
    Let $\gamma_{\Theta,\phi}$ denote the $\iota_\phi^{-1}(\Theta) \to \omega$ geodesic for a given parametrisation $\phi$ and let $(\iota \circ \gamma)_{\Theta,\phi}$ denote its embedding into $\mathcal{M}_\Theta$. Then  there exists some constant $Q$ such that, for any $\phi \in \mathcal{M}_\Phi$, $X,Y \in \mathcal{M}_\Theta$:
    \begin{align}
        \angles{\grad\mathcal{L}(Y), (\iota \circ \gamma)_{Y,\pi(X,Y,\phi)}'(Y)} \\ \geq Q\angles{\grad\mathcal{L}(Y), P_{X \to Y}[(\iota \circ \gamma)_{X,\phi}'(X)]} \nonumber
    \end{align}
\end{assumption}

We ensure that clients will always participate with at most finite gaps:

\begin{assumption}[Bounded Staleness]
    \label{ass:stale}
    Suppose an update from client $i$ arrives at time $\tau$, with the local copy of the client model being $\Theta^t$. Then $\mathbb{E}\left[\norm{\Theta^\tau - \Theta^t }\mid\Theta^t\right] \leq S\max_{t' \in [t..\tau-1]} \norm{\gamma_{\phi^{t'}_k}'(0)}$.
\end{assumption}

Note that the above constraint is immediately implied by \textit{client ergodicity} where, as $T \rta \infty$, every client participates infinitely often in the updates, with non-vanishing probability. In the heterogenous client distribution setting, this ergodicity assumption would be required explicitly to ensure convergence of the global loss. 

For completeness, we reproduce the statement of the theorem, with the full definition of the constants $C_{\{1,2,3\}}$:

\setcounter{theorem}{0}
\begin{theorem}[Convergence of \textsc{AsyncManifold}]
    The \textsc{AsyncManifold} algorithm, with no SWA, assumptions as above, and the local learning rate $\eta_l = \mathcal{O}(\frac{1}{\max\{2C_1,\sqrt{T}\}})$, converges with:
    \begin{align}
        \frac{1}{T}\sum^T_{t = 0} \mathbb{E}\norm{\grad \mathcal{L}(\Theta^t)}^2 &\leq \mathcal{O}\left(\frac{\lambda_\text{min}}{Q\eta_g\sqrt{T}}\left[\mathcal{L}(\Theta^0) - \mathbb{E}\mathcal{L}(\Theta^T)\right]\right) \nonumber\\ & \;\;\;+ \mathcal{O}\left(\frac{\lambda_\text{min}}{\sqrt{T}}(C_2 + 2C_3)\right) \nonumber
    \end{align}
    Where:
    \begin{align}
        C_1 &:= \frac{(1 + \sigma^2_2)L_\Phi}{2} - KL_\iota^2(L_\Theta + GC)\left(\frac{1}{6} + \frac{\eta_g(1+\alpha(S^{-1} - 1))}{4Q}\right) \nonumber\\
        C_2 &:= \frac{1}{\beta}\left[(1-\alpha)\bar{\eta}_g L_\Theta SK^2L_\iota^2G^2  + \alpha L_\Theta K^2L_\iota^2G^2\right] \nonumber\\
        C_3 &:=  KL_\Phi\sigma_1^2 \;\;\;\;\;\;\;\;\;\;\;\; \beta := 1+\alpha(S^{-1} - 1) \nonumber
    \end{align}
\end{theorem}
\begin{proof}[Proof of Theorem 1]
    We assume that the manifold parameters are trained by Riemannian SGD on $\mathcal{M}_\Phi$ for $K$ steps against the loss function:
    \begin{align}
        G_{i,\Theta}(\phi) := \int^1_0 \mathcal{L}(\gamma_\phi(t)) \; dt
    \end{align}
    Where $i$ is a given client index and $\gamma_\phi : [0,1] \to \mathcal{M}_\Theta$ is the constant-speed (scaled) geodesic connecting $\omega$ and $\iota_\phi^{-1}(\Theta^t)$, embedded under $\iota_\phi$. Notice that, by L-smoothness of the lifted loss and the fundamental theorem of calculus, this is L-smooth. By Assumption (\ref{ass:lsl}), we may bound the loss at $\phi$ from below:
    \begin{equation}
        \mathcal{L}_i(\gamma_\phi(t)) \geq \mathcal{L}_i(\gamma_\phi(0)) + \angles{\grad \mathcal{L}_i(\gamma_\phi(0)), t\gamma_\phi'(0)} -\frac{L_\Theta + GC}{2}t^2\norm{\gamma'(0)}^2 \label{eqn:lowerb}
    \end{equation}
    Where the $GC$ term comes from the difference $|\mathcal{L}(\exp_{\Theta}(t\gamma'(0))) - \mathcal{L}(\gamma(t))| \leq G\norm{\exp_{\Theta}(t\gamma'(0)) - \gamma(t)}$, which in turn is bounded by $\frac{GC}{2}t^2\norm{\gamma'(0)}^2$ due to Assumption \ref{ass:lipschitz}. Integrating over $t$ to find a bound on $G$:
    \begin{align}
        G_{i,\Theta}(\phi) &\geq \mathcal{L}_i(\gamma_\phi(0)) + \angles{\grad \mathcal{L}_i(\gamma_\phi(0)), \gamma_\phi'(0)}\int^1_0 t \; dt - \frac{L_\Theta}{2}\norm{\gamma_\phi'(0)}^2\int^1_0t^2 \; dt \\
        &= \mathcal{L}_i(\Theta) + \frac{1}{2} \angles{\grad \mathcal{L}_i(\gamma_\phi(0)), \gamma_\phi'(0)} - \frac{L_\Theta + GC}{6}\norm{\gamma_\phi'(0)}^2
    \end{align}
    We can bound the expectation for $\phi_k$:
    \begin{align}
        \mathbb{E}G_{i,\Theta}(\phi_k) &\geq \mathbb{E}\mathcal{L}_i(\Theta) - \underset{\Delta_1}{\underbrace{\left[\frac{1}{2} \mathbb{E}\angles{-\grad \mathcal{L}_i(\Theta), \gamma_{\phi_k}'(0)} + \frac{L_\Theta + GC}{6}\mathbb{E}\left[\norm{\gamma_{\phi_k}'(0)}^2\right]\right]}}
    \end{align} 
    Similarly, we can use the learning procedure to bound $G_\Theta(\phi))$ from above. By smoothness and the bounded variance Assumption \ref{ass:variance}:
    \begin{align}
        \mathbb{E}G_{i,\Theta}(\phi_{k+1}) &\leq G_{i,\Theta}(\phi_k) - \eta_l\angles{-\grad G_{i,\Theta}(\phi_k), \mathbb{E}g_{i,k}} + \frac{\eta_l^2 L_\Phi}{2}\mathbb{E}\left[\norm{g_{i,k}}^2\right] \\
        &\leq G_{i,\Theta}(\phi_k) - \eta_l\norm{\grad G_{i,\Theta}(\phi_k)}^2 + \frac{\eta_l^2 L_\Phi}{2}\left[(1+\sigma_2^2)\norm{\grad G_{i,\Theta}(\phi_k)}^2 + \sigma_1^2\right] \\
        &= G_{i,\Theta}(\phi_k) - \left(\eta_l - \eta_l^2\frac{(1+\sigma_2^2)L_\Phi}{2}\right)\norm{\grad G_{i,\Theta}(\phi_k)}^2 + \eta_l^2\frac{\sigma_1^2 L_\Phi}{2}
    \end{align}
    Telescoping the sum of $G(\phi_k) - G(\phi_{k+1})$ over $[K]$ yields:
    \begin{equation}
        \mathbb{E}[G_{i,\Theta}(\phi_k)] \leq G_{i,\Theta}(\phi_0) - \underset{\Delta_2}{\underbrace{\left(\eta_l - \eta_l^2\frac{(1 + \sigma_2^2)L_\Phi}{2}\right)\sum^K_{k = 0} \mathbb{E}\norm{\grad G_{i,\Theta}(\phi_k)}^2 + \eta_l^2\frac{KL_\Phi\sigma_1^2}{2}}}
    \end{equation}
    Recalling that $\phi_0$ is a point parametrisation, we have that $G_\Theta(\phi_0) = \mathcal{L}(\Theta)$. We can now combine these bounds, noticing that $\mathcal{L}(\Theta) - \Delta_1 \leq \mathcal{L}(\Theta) - \Delta_2$, hence $\Delta_1 \geq \Delta_2$:
    \begin{align}
        \frac{1}{2} \mathbb{E}\angles{-\grad \mathcal{L}_i(\Theta), \gamma_{\phi_k}'(0)} + \frac{L_\Theta + GC}{6}\mathbb{E}\left[\norm{\gamma_{\phi_k}'(0)}^2\right] \geq \left(\eta_l - \eta_l^2\frac{(1+\sigma_2^2)L_\Phi}{2}\right)\sum^K_{k = 0} \mathbb{E}\norm{\grad G_{i,\Theta}(\phi_k)}^2 - \eta_l^2\frac{KL_\Phi\sigma_1^2}{2} \label{eqn:32}
    \end{align}
    We can now apply the smoothness of $\mathcal{L}$ on $\mathcal{M}_\Theta$ to yield an upper bound in similar form to (\ref{eqn:lowerb}):
    \begin{align}
        \mathbb{E}\mathcal{L}(\Theta^{t+1}) &\leq \mathcal{L}(\Theta^t) - Q\underset{T_1}{\underbrace{\eta_g\mathbb{E}\Sigma^s(\alpha)\angles{-\grad\mathcal{L}_i(\Theta^t), P_{\Theta^t \rta \Theta^s}[\gamma'_{\phi^t_k}(0)]}}} +\eta_g^2\frac{L_\Theta + GC}{2}\mathbb{E}\left[\norm{\gamma'_{\phi^t_k}(0)}^2\right] \label{eqn:18} \\
        \text{where } \Sigma^s(\alpha) &:= 1 +\alpha\max\left[\frac{\norm{\gamma'(0)_{\phi^t_k}}}{\norm{\Theta^s - \Theta^t}} - 1, s - 1 \right]
    \end{align}
    Where we convert to a parallel transport term with Assumption \ref{ass:quality}. Rearranging $T_1$:
    \begin{align}
        T_1 &= \bar{\eta}_g\Sigma^s(\alpha)\angles{-\grad\mathcal{L}_i(\Theta^s) + P_{\Theta^t\to\Theta^s}[\grad \mathcal{L}_i(\Theta^t)] - P_{\Theta^t\to\Theta^s}[\grad \mathcal{L}_i(\Theta^t)], \mathbb{E}P_{\Theta^t \rta \Theta^s}[\gamma'_{\phi^t_k}(0)]} \nonumber \\
        &\geq \bar{\eta}_g\Sigma^s(\alpha)\angles{- P_{\Theta^t\to\Theta^s}[\grad \mathcal{L}_i(\Theta^t)], \mathbb{E}P_{\Theta^t \rta \Theta^s}[\gamma'_{\phi^t_k}(0)]} \nonumber\\
        & \;\;\;\; + \bar{\eta}_g\Sigma^s(\alpha)\angles{-\grad\mathcal{L}_i(\Theta^s) + P_{\Theta^t\to\Theta^s}[\grad \mathcal{L}_i(\Theta^t)], \mathbb{E}P_{\Theta^t \rta \Theta^s}[\gamma'_{\phi^t_k}(0)]} \nonumber\\
        &= \bar{\eta}_g\Sigma^s(\alpha)\angles{-\grad \mathcal{L}_i(\Theta^t), \bar{\eta}_g\mathbb{E}\gamma'_{\phi^t_k}(0)} - \bar{\eta}_g\underset{T_2}{\underbrace{\Sigma^s(\alpha)\angles{\grad\mathcal{L}_i(\Theta^s) - P_{\Theta^t\to\Theta^s}[\grad \mathcal{L}(\Theta^t)], \mathbb{E}P_{\Theta^t \rta \Theta^s}[\gamma'_{\phi^t_k}(0)]}}} \label{eqn:21}
    \end{align}
    We choose the global learning rate $\eta_g$ to ensure that $\norm{\eta_g\gamma'_{\phi_k}(0)} = \norm{\bar{\eta}_g\exp^{-1}_{\Theta}(\iota(\phi_k,\omega))}$. By the Lipschitz property of embedded diameter and the fact that $\phi_0$ is a point parametrisation, we have that:
    \begin{align}
        \norm{\exp^{-1}_\Theta(\iota(\phi_k,\omega))} \leq L_{\iota}\norm{\exp^{-1}_{\phi_0}(\phi_k)} \leq L_\iota\eta_l\sum^K_{k = 0} \norm{\grad G_{i,\Theta}(\phi_k)} \label{eqn:22}
    \end{align}
    Where the last inequality is by the geodesic triangle and AM-GM inequalities. This enables us to continue bounding $T_2$:
    \begin{align}
        T_2 &\leq \Sigma^s(\alpha(\norm{\grad\mathcal{L}_i(\Theta^s) - P_{\Theta^t\to\Theta^s}[\grad \mathcal{L}_i(\Theta^t)]} \cdot \norm{\mathbb{E}P_{\Theta^t \rta \Theta^s}[\gamma'_{\phi^t_k}(0)]} \\
        &\leq \left((1-\alpha) + \alpha\frac{\norm{\gamma'(0)_{\phi^t_k}}}{\norm{\Theta^s - \Theta^t}}\right)L_\Theta\norm{\Theta^s - \Theta^t}\cdot \norm{\mathbb{E}\gamma'_{\phi^t_k}(0)} \\
        &\leq (1-\alpha)\left[L_\Theta \sum_{i \in [t..s]} \eta_g\norm{\gamma'_i(0)}\norm{\mathbb{E}\gamma'_s(0)}\right] + \alpha \norm{\mathbb{E}\gamma'_{\phi^t_k}(0)}^2 \\
        &\leq (1-\alpha)\left[L_\Theta \sum_{i \in [t..s]}\left[\bar{\eta}_g KL_\iota^2\eta_l^2\sum_{k \in[K]} \norm{\grad G_{i,\Theta^i}(\phi^i_k)}^2\right]\right] + \alpha \norm{\mathbb{E}\gamma'_{\phi^t_k}(0)}^2 \\
        &\leq (1-\alpha)\eta_l^2\bar{\eta}_gL_\Theta SK^2L_\iota^2G^2  + \alpha\eta_l^2L_\Theta K^2L_\iota^2G^2\label{eqn:25}
    \end{align}
    Substituting (\ref{eqn:25}) into (\ref{eqn:21}), then accumulating into (\ref{eqn:18}) along with (\ref{eqn:32}):
    \begin{align}
        \mathbb{E}\mathcal{L}(\Theta^{t + 1}) &\leq \mathcal{L}(\Theta^t) - 2Q\eta_g\Sigma^s(\alpha)\left(\eta_l - \eta_l^2\frac{(1 + \sigma^2_2)L_\Phi}{2}\right)\sum^K_{k = 0} \mathbb{E}\norm{\grad G_{i,\Theta^t}(\phi_k^t)}^2 \\
        & \;\;\; + \left(\eta_g^2\frac{L_\Theta + GC}{2} + 2Q\Sigma^s(\alpha)\eta_g\frac{L_\Theta + GC}{6}\right)\mathbb{E}\norm{\gamma'_{\phi^t_k}(0)}^2 \\
        & \;\;\; + 2Q\Sigma^s(\alpha)\eta_g\eta_l^2\frac{KL_\Phi \sigma_1^2}{2} + Q\eta_gT_2 \\
        &\leq \mathcal{L}(\Theta^t) + 2Q\Sigma^s(\alpha)\eta_g\eta_l^2\underset{C_3}{\underbrace{KL_\Phi \sigma_1^2}} \nonumber\\
        &\;\;\;\; + Q\bar{\eta}_g\eta_l^2\Sigma^s(\alpha)\underset{C_2}{\underbrace{\frac{1}{1+\alpha(S^{-1} - 1)}\left[(1-\alpha)\bar{\eta}_g L_\Theta SK^2L_\iota^2G^2  + \alpha L_\Theta K^2L_\iota^2G^2\right]}} \nonumber\\
        &\;\;\;\; - 2Q\eta_g\eta_l\Sigma^s(\alpha)\left(1 - \eta_lC_1\right)\sum^K_{k = 0} \mathbb{E}\norm{\grad G_{i,\Theta^t}(\phi_k^t)}^2 \\
        \text{where } C_1 &:= \frac{(1 + \sigma^2_2)L_\Phi}{2} - KL_\iota^2(L_\Theta + GC)\left(\frac{1}{6} + \frac{\eta_g(1+\alpha(S^{-1} - 1))}{4Q}\right)
    \end{align}
    
    We rearrange and telescope over $[T]$ to find a convergence bound in terms of the Riemannian gradient on $\mathcal{M}_\Phi$:
    \begin{align}
        \frac{1}{T}\sum^T_{t = 0} \mathbb{E}\norm{\grad G_{i,\Theta^t}(\phi^t)}^2 &\leq \frac{\mathcal{L}(\Theta^0) - \mathbb{E}\mathcal{L}(\Theta^T)}{2TQ\eta_l\eta_g(1+\alpha(S^{-1} - 1))\left(1 - \eta_lC_1\right)} + \eta_l\frac{C_2 + 2C_3}{2(1-\eta_lC_1)} \label{eqn:27}
    \end{align}
    We need now to translate this to a bound w.r.t. $\mathcal{M}_\Theta$. Recalling that the differential (and hence its adjoint) are linear operators, by a standard linear algebraic argument we have:
    \begin{equation}
        \norm{(D\iota_\omega)^*_\phi(v)}^2 = \angles{(D\iota_\omega)^*_\phi(v), D\mathcal{L}_\Theta(v)} = \angles{v, ((D\iota_\omega)(D\iota_\omega)^*)_\phi(v)} \geq \lambda_\text{min}\norm{v}^2
    \end{equation}
    For $\lambda_\text{min}$ the smallest absolute eigenvalue of $D(\iota_\omega)_\Theta$. Accordingly:
    \begin{equation}
        \norm{\grad G_{i,\Theta^t}(\phi^t)}^2 = \norm{D(\iota_\omega)^*_\phi[\grad \mathcal{L}_i(\Theta^t)]}^2 \geq \frac{1}{\lambda_\text{min}}\norm{\grad \mathcal{L}_i(\Theta^t)}^2 \label{eqn:lambdabound}
    \end{equation}
    We substitute (\ref{eqn:lambdabound}) into (\ref{eqn:27}), simply multiplying by $\lambda_\text{min}$ (bounded above zero by Assumption \ref{ass:immersion}). 

    Notice that we have used $\mathcal{L}$ without considering the proximal term in this analysis. This is because our result bounds the loss gradient on $\mathcal{M}_\Theta$ at $\Theta^t$ by bounding the loss gradient on $\phi$ at $\phi^t_\text{init}$ - hence the proximal and raw losses coincide when evaluated at this point, so we can conclude a bound on the raw loss immediately, although we have technically abused notation referring to the client optimising over $\mathcal{L}$. The result then follows from an appropriate choice for $\eta_l$.
\end{proof}

\section{EXPERIMENTAL DETAILS}
\label{sec:expdet}

Experiments were run on two Nvidia RTX GPUs (1x 5070, 1x 3070), each simulating 15 clients. The scheduler accurately simulates varying asynchronous processing speeds by stochastically choosing clients to run from the queue according to expected length of local training - in our case primarily influenced by local dataset size - and current waiting time. Updates are processed on a central server thread and clients immediately dispatched back to the waiting pool with updated model weights. 

For each method implemented, we use a local Adam optimiser on the \textsc{FedProx} objective for 2 epochs with $\eta_l = \mu = 0.001$, only tuning global parameters of the aggregation framework.

The most influential hyperparameter is the choice of global learning rate $\eta_g$, which for all methods was found by line search over $\{0.25,0.5,1.0,1.5,2.0,2.5,3.0,5.0\}$ - see Table \ref{fig:lrtab} for the choices by method and dataset.

\begin{table}[ht]
  \centering
  \begin{minipage}[b]{\textwidth}
    \centering
    \begin{tabular}{c | c | c | c}
        \Xhline{2\arrayrulewidth}
         \multicolumn{4}{c}{\rule{0pt}{3ex}\textbf{Global Learning Rates ($\eta_g$)} \rule[-1.5ex]{0pt}{0pt}}\\
        \Xhline{2\arrayrulewidth}
        \rule{0pt}{3ex}
        \textbf{Method} & \textbf{FEMNIST} & \textbf{Shakespeare}  & \textbf{CXR8}  \rule[-1.5ex]{0pt}{0pt}\\
        \hline
        \rule{0pt}{3ex}
        \textsc{FedAsync} & 3.0 & 5.0 & 2.0 \\
        \textsc{FedOrtho} & 3.0 & 5.0 & 2.0 \\
        \textsc{FedGS} & 1.0 & 2.5 & 1.0 \\
        \textsc{DC-ASGD} & 1.0 & 2.5 & 1.0 \\
        \textsc{FedBuff} & 1.0 & 2.0 & 1.0 \\
        \textsc{AsyncFedED} & 0.25 & 1.5 & 0.5 \\
        \textsc{Proposed} & 0.5 & 1.5 & 0.5 \\
        \textsc{ProposedED} & 0.25 & 1.0 & 0.25 \\
        \Xhline{2\arrayrulewidth}
    \end{tabular}
  \end{minipage}

  \vspace{1em}
  \caption{Macro AUROC and AUPRC scores for each method across the 8 conditions in the CXR8 dataset.}
  \label{fig:lrtab}
\end{table}

\subsection{FEMNIST}

805,263 28$\times$28 black-and-white images, representing a single alphanumeric character (hence one of 62 classes). Samples were heterogenously partitioned into 30 clients according to the Dirichlet distribution ($\alpha = 0.5$) on class labels. 

Figure \ref{fig:femnistarch} shows the full CNN architecture used for this dataset (ReLU activations not shown).

\begin{figure}[h]
    \centering
    \includegraphics[width=0.35\linewidth]{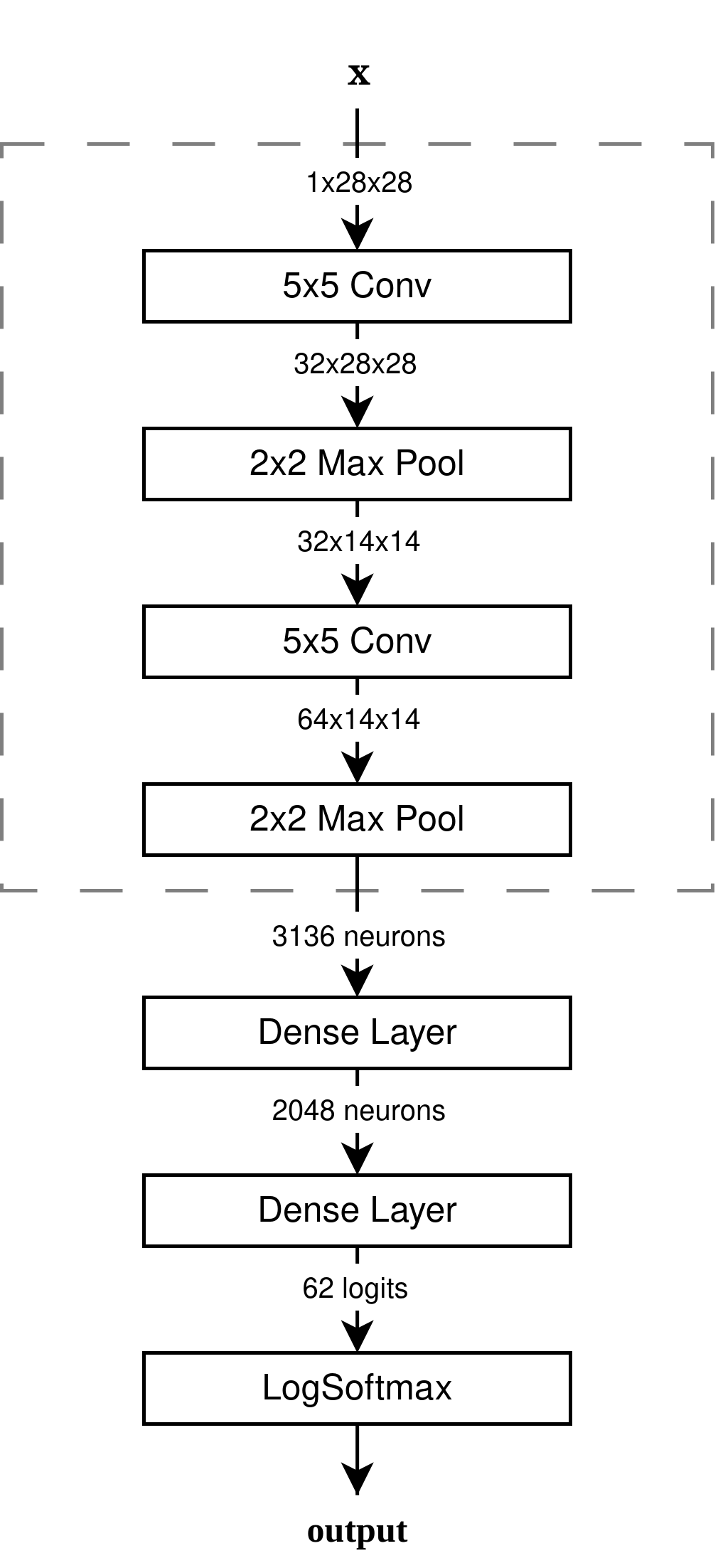}
    \caption{CNN Architecture for FEMNIST}
    \label{fig:femnistarch}
\end{figure}

Contributions from each client are weighted by proportion of dataset seen by that client. For DC-ASGD, the $\lambda_t$ parameter is set dynamically with $\lambda_0 = 2.0$, as proposed by \cite{dcasgd}. For \textsc{FedBuff}, we use $K = 10$ as recommended in \cite{fedbuff}. 

For \textsc{AsyncFedED}, we follow the original paper \citep{asyncfeded}, and use $\overline{\gamma} = 1.0, \kappa = 1$ (notice that $\lambda$ in their notation is subsumed by $\eta_g$ in ours). Increasing $\bar{gamma}$ to above 1 increases training stability, but increases wall-clock time far more and results in worse performance in communication round terms. We note that in the early stages, staleness computed according to their Equation (6) can exhibit high variance that can throw training off. Accordingly, we do not compute staleness dynamically until after a short ``warm-up'' period, using the modified:
\begin{equation}
    \tilde{\gamma}(i,t) = \begin{cases}
        \gamma(i,t) & t > 10 \\
        \overline{\gamma} & \text{otherwise}
    \end{cases}
\end{equation}
\textsc{AsyncFedED} is unique among methods tested in using adaptive per-client epoch counts. All our convergence rate results are computed according to communication round count as opposed to wall-clock time, but we do not notice much advantage given to the method, which achieves similar results to other baselines when measured according to communication rounds, despite taking far greater wall-clock time than \textsc{FedAsync}. We can possibly attribute this to the reduced performance of the \textsc{FedAsync} update rule as the number of local epochs increases outweighing any task-balancing issues.

For \textsc{AsyncBezier}, we set $\vartheta = 1$, using the ``orthogonalising'' version of the \textsc{OrthoDC} update rule.

\subsection{Shakespeare}

The dialogue lines are first separated by speaker and then windowed into 80-character sequences, for a total of 4,027,181 samples drawn from 35 plays. We allocate each play wholly to a distinct client - since there are 30 clients, 5 will receive 2 plays each, simulating real-world clients which have a disproportionate share of the samples.

We use the nanoGPT framework [\url{https://github.com/karpathy/nanoGPT}] to build a GPT-2 like character-level transformer  with 6 layers, 6 heads, a 128-dimensional embedding, and dropout $p=0.1$. We train for next-character prediction given an 80-character input sequence. Most (non-LR) hyperparameters remain the same:

For \textsc{AsyncFedED} we maintain $\overline{\gamma} = 1$, which gives far superior performance when compared to $\overline{\gamma} = 3$ (and at faster wall-clock).

For \textsc{AsyncBezier}, we instead set $\vartheta = 0$, using the ``gradient surgery'' version of the \textsc{OrthoDC} update rule.

\subsection{CXR8}

This is a dataset of 112,120 $128 \times 128$ black-and-white chest X-Ray images. 8 conditions (\textit{Atelectasis}, \textit{Cardiomegaly}, \textit{Effusion}, \textit{Infiltration}, \textit{Mass}, \textit{Nodule}, \textit{Pneumonia}, \textit{Pneumothorax}) are labelled for and the model is trained to detect their presence, encoded as a multi-hot vector to allow for co-incidence. The data is drawn from scans of 30,805 patients, with each assigned wholly to one of 30 clients.

For CXR8, we use the ShuffleNetv2 architecture \citep{shufflenet}, expanded to the $\times 1.5$ version. We use the weights available from PyTorch (\url{https://docs.pytorch.org/vision/main/models/generated/torchvision.models.shufflenet_v2_x1_5}) which have been pre-trained on the general-purpose ImageNet dataset. The CXR8 images are then rescaled to $128 \times 128$ and reshaped to 3 channels in order to match ImageNet input before being used to fine-tune the model. 

$\vartheta$ remains $= 0$ for \textsc{AsyncBezier} and $\overline{\gamma} = 1$ for \textsc{AsyncFedED}.

\subsection{Fairness Calculations}
\label{sec:faircalc}

\newcommand{\abs}[1]{\left|#1\right|}
For completeness, we provide the method to compute the Gini Coefficient and Theil Index as used in Figure \ref{fig:fairnessplot}; both definitions are sourced from \cite{econineq}. The Gini Coefficient is a measure of pairwise variance in a sample $X = \{x_1,...,x_N\}$, normalised by the sample mean $\bar{X}$:
\begin{equation}
    \text{Gini}(X) := \frac{1}{2N^2\bar{X}}\sum_{i \in[N]}\sum_{j \in [N]} \abs{x_i - x_j}
\end{equation}
Intuitively, it measures the difference in area between the plot of cumulative relative ``wealth'' (here, the values of $x_i$) against cumulative proportion of the population for the observed sample and the plot that would be yielded from the uniform distribution between minimum and maximum values (a straight line).

The Theil Index is a measure derived instead from information theory, quantifying the difference between the Shannon entropy of the observed distribution of proportional ``wealth'' and the entropy of the same uniform distribution:
\begin{equation}
    \text{Theil}(X) := \frac{1}{N\bar{X}} \sum_{i \in [N]} x_i \log \left(\frac{x_i}{\bar{X}}\right)
\end{equation}
We note that these are both simply measures of concentration for $X$'s distribution, but this is a valid proxy for inequality as distance from the ``most equal'' uniform distribution.

\end{document}